\documentclass{article}

\usepackage{microtype}
\usepackage{graphicx}
\usepackage{subcaption}
\usepackage{booktabs} 
\usepackage{multirow, makecell}
\usepackage{multirow}
\usepackage[dvipsnames]{xcolor}

\usepackage{hyperref}

\usepackage[preprint]{icml2026}

\usepackage{amsmath}
\usepackage{amssymb}
\usepackage{mathtools}
\usepackage{amsthm}

\usepackage[capitalize,noabbrev]{cleveref}

\theoremstyle{plain}
\newtheorem{theorem}{Theorem}[section]
\newtheorem{proposition}[theorem]{Proposition}
\newtheorem{lemma}[theorem]{Lemma}

\theoremstyle{definition}
\newtheorem{definition}[theorem]{Definition}

\theoremstyle{remark}

\usepackage[textsize=tiny]{todonotes}

\icmltitlerunning{Jacobian Regularization Stabilizes Long-Term Integration of Neural Differential Equations}

\begin{document}

\twocolumn[
  \icmltitle{Jacobian Regularization Stabilizes Long-Term Integration  \\
  of Neural Differential Equations}



  \icmlsetsymbol{equal}{*}

  \begin{icmlauthorlist}
    \icmlauthor{Maya Janvier}{locean}
    \icmlauthor{Julien Salomon}{inria,ljll}
    \icmlauthor{Etienne Meunier}{locean,inria}
  \end{icmlauthorlist}

  \icmlaffiliation{locean}{LOCEAN, Sorbonne University, Paris, France}
  \icmlaffiliation{inria}{ANGE team, INRIA, Paris, France}
  \icmlaffiliation{ljll}{LJLL, Sorbonne University, Paris, France}

  \icmlcorrespondingauthor{Maya Janvier}{maya.janvier@locean.ipsl.fr}

  \icmlkeywords{neural differential equations, neural ordinary differential equations, stabilization, dynamical systems, dynamics, conservation laws, scientific machine learning, physics-informed machine learning}

  \vskip 0.3in
]



\printAffiliationsAndNotice{}  

\begin{abstract}
    Hybrid models and Neural Differential Equations (NDE) are getting increasingly important for the modeling of physical systems, however they often encounter stability and accuracy issues during long-term integration. Training on unrolled trajectories is known to limit these divergences but quickly becomes too expensive due to the need for computing gradients over an iterative process. In this paper, we demonstrate that regularizing the Jacobian of the NDE model via its directional derivatives during training stabilizes long-term integration in the challenging context of short training rollouts. We design two regularizations, one for the case of known dynamics where we can directly derive the directional derivatives of the dynamic and one for the case of unknown dynamics where they are approximated using finite differences. Both methods, while having a far lower cost compared to long rollouts during training, are successful in improving the stability of long-term simulations for several ordinary and partial differential equations, opening up the door to training NDE methods for long-term integration of large scale systems.
\end{abstract}

\section{Introduction}
\label{sec:intro}

In Neural Ordinary Differential Equations (Neural ODE), a neural network (NN) is used as the dynamical function or right-hand side of an ODE and is integrated through a numerical solver to produce the next step of the trajectory of a physical system \cite{chen2019neuralordinarydifferentialequations}. This extension of ResNets to other solvers has gained in popularity across different deep learning tasks: not only the modeling of physical systems \cite{lippe2023pderefinerachievingaccuratelong, white2024stabilizedneuraldifferentialequations, Yin_2021, Linot_2023, pal2025semiexplicitneuraldaeslearning}, but also for classification tasks \cite{yan2022robustnessneuralordinarydifferential, kang2021stableneuralodelyapunovstable} and generative tasks \cite{finlay2020trainneuralodeworld, song2021scorebasedgenerativemodelingstochastic}. For all of these tasks, the robustness and generalization ability of the models are very important to ensure stable and accurate long-term predictions in the case of physical modeling, and resilience to adversarial attacks in general. The Jacobian is a crucial mathematical concept in stability analysis \cite{Margazoglou_2023} and has been leveraged in many different ways for deep learning applications. 

In the context of adversarial attacks, the upper bound on the norm of the Jacobian, the Lipschitz constant, has sparked a lot of interest: a small Lipschitz constant guaranties that similar inputs in terms of a norm produce similar outputs for the same norm (for example, similar class labels). Unfortunately, estimating computationally efficient Lipschitz constant for training is not an easy task. \citeauthor{scaman2019lipschitzregularitydeepneural} \yrcite{scaman2019lipschitzregularitydeepneural} showed that its exact computation is NP-hard even for 2-layer networks and developed AutoLip and SeqLip to accelerate its estimation. Another issue they mention is the difficulty not to overestimate the Lipschitz constant: \citeauthor{pintore2024computablelipschitzboundsdeep} \yrcite{pintore2024computablelipschitzboundsdeep} build accurate upper bounds of it but they remain expensive compute. 

We see that there is a strong balance between accuracy and efficiency in this task, that authors tackle differently: some compute a generalized jacobian via means of the chain rule to estimate the Lipschitz constant \cite{jordan2021exactlycomputinglocallipschitz} while others leverage semi-definite relaxations and Integral Quadratic Constraints (IQC) \cite{raghunathan2018semidefiniterelaxationscertifyingrobustness, fazlyab2023efficientaccurateestimationlipschitz}. All these methods regularize the NN with their estimation of the Lipschitz constant with fixed penalization or more complex optimization of it \cite{Pauli_2022}. Others directly modify the architecture of the NN by building Lipschitz layers \cite{meunier2022dynamicalperspectivelipschitzneural, wang2023directparameterizationlipschitzboundeddeep} or new networks to satisfy IQC for system identification problems \cite{revay2023recurrentequilibriumnetworksflexible}. TisODE \cite{yan2022robustnessneuralordinarydifferential} leverages both approaches by removing the time dependence of the dynamics and enforcing the convergence to a steady state for robustness. 

Other works focused on the eigenvalues of the Jacobian to enforce different stability properties: \citeauthor{kang2021stableneuralodelyapunovstable} \yrcite{kang2021stableneuralodelyapunovstable} forced the Jacobian of the ODE in the Neural ODE layer to have eigenvalues with negative real parts at the Lyapunov-stable equilibrium points, while \citeauthor{westny2024stabilityinformedinitializationneuralordinary} \yrcite{westny2024stabilityinformedinitializationneuralordinary} showed that ensuring that the Jacobian eigenvalues were within the stability region of the integration scheme (Euler, Runge-Kutta schemes) was beneficial for the training of model. Jacobian regularization as an analog to gradient regularization \cite{novak2018sensitivitygeneralizationneuralnetworks} has been demonstrated to accelerate training for Neural ODE based generative models \cite{finlay2020trainneuralodeworld}.

Neural ODE instabilities are often encountered when doing long-rollouts of physical systems \cite{frezat2021posteriorilearningquasigeostrophicturbulence, list2024differentiabilityunrolledtrainingneural}, which are crucial for systems such as climate models. The issue was mainly approached from an energetic point of view until now: \citeauthor{Linot_2023} \yrcite{Linot_2023} introduced dissipation in the Neural ODE formulation with a linear term. As solutions to the equations sometimes lay in certain known manifolds (in the form of conservation laws for example), several works regularize \cite{white2024stabilizedneuraldifferentialequations} or project the Neural ODE onto this manifold during training \cite{pal2025semiexplicitneuraldaeslearning}. Both methods leverage products of the Jacobian in order to achieve their stabilization. The PDE Refiner method combines denoising models and Neural ODE to iteratively refine the solution in order to learn across all frequencies and not only the dominant one \cite{lippe2023pderefinerachievingaccuratelong}.

In our work, we tackle long-term instabilities in Neural ODE with Jacobian regularizations involving directional derivative estimations, in the spirit of adversarial attack works. This is, to the best of our knowledge, the first attempt of this kind in the modeling of physical systems. We first present the challenges of accurate long rollouts for Neural ODEs and their particular link with the Jacobian \ref{sec:challenge}, then present our methods for estimating Jacobian norms at a reasonable cost (\ref{sec:method}). We validate our method on autonomous systems (two ODE and one PDE) (\ref{sec:results}) and finally discuss the results and next challenges (\ref{sec:conclu}). 

\section{Challenges of Accurate Long Rollouts for Neural Differential Equations}
\label{sec:challenge}

Let us consider a general Partial Differential Equation (or PDE) in the form of 
\begin{equation}
    \frac{\partial \textbf{x}}{\partial t}=F(t, \textbf{x}, \nabla \textbf{x}, \nabla ^2 \textbf{x}, ...)
\end{equation}
with \textbf{x} the prognostic variables evolving temporally through the dynamic function $F$. To simulate trajectories of \textbf{x}, we numerically integrate the PDE forward in time using a temporal integration scheme or solver (e.g. Euler forward, Runge-Kutta 4) with a time-step $\Delta t$. However, we usually have an incomplete knowledge of $F$ and still want to be able to simulate trajectories. The concept of Neural ODE \cite{chen2019neuralordinarydifferentialequations} is to train a neural network $F_\theta$ from ground truth trajectories to play the role of $F$. The length of those trajectories we feed to the neural network, measured in number of integration steps performed through the solver, is called the training rollout $N$. In this section, we present our approach in the framework of Neural ODE. 

\subsection{Learning through solvers: the influence of rollout}
Neural ODE models \cite{chen2019neuralordinarydifferentialequations} are trained through numerical solvers. Consider a neural network $F_\theta$ representing the right-hand side or dynamics of an (here autonomous) ODE of the form $\frac{dx_\theta(t)}{dt}=F_\theta(x_\theta(t))$, associated with an integration scheme (or solver) and a step-size $\Delta t$, that aims to produce the same trajectories as the true ODE $\frac{dx(t)}{dt}=F(x(t))$, namely $X \in \mathcal{T}_F=\{x_i=x_0+\int_{t_0}^{t_i}F(x(s))ds, 0 \leq i \leq N\}$. The training is carried out by minimizing the trajectory loss \cite{chen2019neuralordinarydifferentialequations, Yin_2021, Linot_2023}:
\begin{equation}
    \mathcal{L}^N_{\text{traj}}(\theta)=\sum_{X\in \mathcal{T}_{F}}\sum_{i=0}^N\Vert x(t_i)-\tilde{x}_\theta(t_i)\Vert_2^2  \text{,} \quad t_i=i\Delta t
\end{equation}
with $\tilde{x}_\theta(t_i)$ obtained by numerically integrating the ODE forward in time for $0 \leq i \leq N$ steps.

In offline training, the Neural ODE is tasked to minimize $\Vert F(x)-F_\theta(x)\Vert_2^2$, so that learning $F$ is consequently achieved in a direct way, without using the solver. However, this method needs access to the true $F$ which can be unknown, and may lead to diverging trajectories for long-term integration at inference in most cases \cite{frezat2021posteriorilearningquasigeostrophicturbulence, list2024differentiabilityunrolledtrainingneural}.

On the other hand, performing the training with $N>1$ (or online training), i.e., through the solver, only requires access to trajectory points (and no longer to $F$) and results in much more stable rollouts \cite{ frezat2021posteriorilearningquasigeostrophicturbulence, list2024differentiabilityunrolledtrainingneural}. However, this enhancement comes at the heavy cost of computing the gradient through the solver for each time-step which quickly becomes a bottleneck for long training rollouts $N$. The interest of long rollouts is also limited in practice as errors in gradient estimation accumulate and result in unstable training \cite{list2024differentiabilityunrolledtrainingneural}.

Experimentally, Neural ODE trained on small $N$ are not stable during inference. There is thus a need to produce stable long-term Neural ODE simulations based on short training rollouts, especially as long rollouts would simply become too costly for more complex and high-dimensional problems. 

\subsection{Jacobian and long-term stability}
The Jacobian of a function $F: \mathbb{R}^n \rightarrow \mathbb{R}^m$ is defined as the matrix $J_F$ of all its first-order partial derivatives, and is involved in stability results through its Lipschitz constant.

\begin{definition}
A function $F$ is Lipschitz continuous if there exists $L \in \mathbb{R}^+$ such that $\forall x,y \in \mathbb{R}^n \times \mathbb{R}^n$:
\begin{equation*}
    \Vert F(x)-F(y)\Vert_2 \leq L \Vert x-y\Vert_2 .
\end{equation*}
$L$ is called a Lipschitz constant of $F$ and when it is differentiable, the optimal (smallest) Lipschitz constant is defined as $L_F = \sup_{x\in \mathbb{R}^n} \Vert J_F(x)\Vert_2$. 
    
\end{definition}

This Lipschitz constant is directly related to the long-term stability of Neural ODEs. Let us consider the true continuous trajectory of $x$ from the true dynamics $F$ and the predicted $x_\theta$ from the learned continuous dynamics $F_\theta$, with the same initial conditions, i.e., $x(t=0)=x_\theta(t=0)$. The offline error between $F$ and $F_\theta$ is denoted $\epsilon_\theta=F_\theta-F$. Using Grönwall's lemma we get for the trajectory error:

\begin{proposition} Trajectory error bound

For $t>0$, 
\begin{equation}
    \Vert x(t)-x_\theta(t)\Vert_2 \leq \frac{\Vert\epsilon_\theta\Vert_\infty}{L_{F_\theta}} (\exp (L_{F_\theta}t)-1)
\end{equation}
with $L_{F_\theta}$ the Lipschitz constant of $F_\theta$.
\end{proposition}

\begin{proof}
    See appendix \ref{a:lipschitz}.
\end{proof}

Thus, a model with a low $\epsilon_\theta$ but a large $L_{F_\theta}$ can lead to diverging trajectories as $t$ grows, driven by the exponential term in the upper bound as already documented in the literature \cite{Yin_2021,yan2022robustnessneuralordinarydifferential}. On the other hand, a small Lipschitz constant $L_{F_\theta}$ limits the rate of divergence from the true trajectory, but if it becomes too small it will harm the ability of the model to correctly learn the dynamics. As noted in several works \cite{huang2021trainingcertifiablyrobustneural, wang2022quantitativegeometricapproachneuralnetwork, fazlyab2023efficientaccurateestimationlipschitz}, there exists a trade-off between the smoothness of a neural network and its expressiveness, i.e., between $\epsilon_\theta$ and $L_{F_\theta}$. Finding an appropriate balance between the two is still an open research question. 

In this paper, we propose a method to tackle this challenge by minimizing $\Vert J_F(x)-J_{F_\theta}(x)\Vert_2$. Not only will we better learn the dynamic, but we demonstrate below that it allows us to obtain physically informed bounds on the Lipschitz constant of $F_\theta$:

\begin{proposition}
\label{prop:loss}
    If there is $\epsilon>0$ such that 
    \begin{equation*}
        \forall x\in \mathbb{R}^n, \Vert J_F(x)-J_{F_\theta}(x)\Vert_2 \leq \epsilon,
    \end{equation*}
then we have 
    \begin{equation*}
        L_F-\epsilon \leq L_{F_\theta} \leq L_F + \epsilon.
    \end{equation*}
\end{proposition}

\begin{proof}

Using the triangular inequality 
\begin{equation*}
    \Vert \Vert{a} \Vert_2-\Vert b \Vert_2 \Vert_2 \leq \Vert a-b \Vert_2,
\end{equation*}
we have 
\begin{equation*}
    \Vert\Vert J_F(x)\Vert_2 -\Vert J_{F_\theta}(x) \Vert_2 \Vert_2\leq \Vert J_F(x)-J_{F_\theta}(x)\Vert_2 \leq \epsilon,
\end{equation*}
which is equivalent to
\begin{equation*}
    \Vert J_F(x)\Vert_2 - \epsilon \leq \Vert J_{F_\theta}(x) \Vert_2 \leq \Vert J_F(x)\Vert_2 + \epsilon.
\end{equation*}
Since these inequalities hold all $x\in \mathbb{R}^n$, we finally have  
\begin{equation*}
    \sup_x\Vert J_F(x)\Vert_2 - \epsilon \leq \sup_x\Vert J_{F_\theta}(x) \Vert_2 \leq \sup_x\Vert J_F(x)\Vert_2 + \epsilon,
\end{equation*}

i.e., $$L_F-\epsilon \leq L_{F_\theta} \leq L_F + \epsilon. $$
\end{proof}

Proposition \ref{prop:loss} shows that if $J_F$ and $J_{F_\theta}$ are close, then their Lipschitz constant will also be close, in particular the Lipschitz constant of $F_\theta$ will be bounded. Since correctly estimating Lipschitz constants of neural networks is an open research question \cite{jordan2021exactlycomputinglocallipschitz, fazlyab2023efficientaccurateestimationlipschitz, pintore2024computablelipschitzboundsdeep}, this method offers an alternative to constrain $L_{F_\theta}$ while staying close to the true dynamic of the system.

This method is further motivated by the result of a numerical experiment presented in Table \ref{tab:lip_diff}: compared to models trained on short rollouts, those trained
on longer rollouts sacrifice some precision on the offline error $\epsilon_\theta$ for better generalization on $\mathcal{L}_{\text{traj}}$ and better learning of the Jacobian. We argue that this is an explanation to the huge gap in performance between offline and online training reported in several works \cite{frezat2021posteriorilearningquasigeostrophicturbulence, list2024differentiabilityunrolledtrainingneural}. 

\section{Methods}
\label{sec:method}

We have seen that the Jacobian represents a fundamental tool for controlling long-term stability. Our goal is to produce stable integrations in the challenging context of short training rollouts $N$. 

In what follows, we present a regularization strategy that improves the stability of Neural Differential Equations trained on short rollouts, based on a regularization term that includes the $L^2$ distance between the Jacobian $J_{F_\theta}$ of the Neural ODE and $J_F$ the one of the real system, on true trajectories points: 
\begin{equation}
\label{eq:traj_loss}
    \mathcal{L}(\theta) = \mathcal{L}^N_{\text{traj}}(\theta)+\lambda \sum_{X\in \mathcal{T}_F}\sum_{i=0}^N\Vert J_F(x(t_i))-J_{F_\theta}(x(t_i))\Vert_2^2.
\end{equation}

The constraint on the Jacobian we are trying to enforce here is precisely the term introduced in Proposition \ref{prop:loss}: we exploit the link between the Jacobian and the Lipschitz constant of the system to constrain them both in a physically informed manner. We adopt a regularization (or penalization) strategy by weighting the Jacobian constraint with a fixed hyperparameter $\lambda$, as done in \citeauthor{yan2022robustnessneuralordinarydifferential}.  

\subsection{Estimating the norm of Jacobians efficiently}
Such strategy will require computing the Jacobians  of $J_{F_\theta}$ and $J_F$. As any neural network is differentiable through forward-mode Automatic Differentiation (AD), these Jacobian and their Fröbenius norms can be computed directly in practice. However, it remains very expensive and inefficient to evaluate this quantity during training as we have to compute many matrices only to use their norms. 

We can reduce the dimensionality of the estimation by leveraging the relationship between the Jacobian and the directional derivative. Indeed, using the Hutchinson estimator of the trace of a matrix \cite{Hutchinson01011989}, we have:
\begin{equation}
\label{eq:dd}
\begin{aligned}
    \Vert J_F(x) - J_{F_\theta}(x)\Vert_2^2 &= \Vert J_{F-F_\theta}(x)\Vert_2^2 \\
    &= \Vert J_{\epsilon_\theta}(x)\Vert_2^2\\
    &= Tr(J_{\epsilon_\theta}(x)^TJ_{\epsilon_\theta}(x)) \\
        &\approx \mathbb{E}_{v \sim P(v)} [v^T \cdot (J_{\epsilon_\theta}(x)^TJ_{\epsilon_\theta}(x)) \cdot v]\\
        &\approx \mathbb{E}_{v \sim P(v)}[(J_{\epsilon_\theta}(x) \cdot v)^T (J_{\epsilon_\theta}(x) \cdot v)] \\ 
     & \approx \mathbb{E}_{v \sim P(v)}[\Vert J_{\epsilon_\theta}(x) \cdot v\Vert_2^2] 
\end{aligned}
\end{equation}
where $J_{\epsilon_\theta}$ denotes the Jacobian of the error $\epsilon_\theta$. 

The choice for $P(v)$ is usually the standard normal distribution, but any distribution following $\mathbb{E}_{v \sim P(v)}[v^Tv]=I$ is suitable (e.g., Rademecher for minimal variance \cite{Hu_2024}), and the Jacobian norms can be estimated by Monte-Carlo using $v_i$ \textit{i.i.d} sampled from $P(v)$. The Hutchinson trace estimator succeeds in reducing the dimension of the norms we are computing as we do not need to evaluate the full Jacobian anymore but only Jacobian-vector products (JVP).

\begin{proposition}
\label{prop:jvp}
Let $v$ be a vector in $\mathbb{R}^n$ and $h$ a positive real number, we have 
\begin{equation*}
    J_{\epsilon_\theta}(x) \cdot v =D_v \epsilon_\theta(x) := \lim_{h \rightarrow 0} \frac{\epsilon_\theta(x+hv)-\epsilon_\theta(x)}{h},
\end{equation*}
where $D_v \epsilon_\theta(x)$ is the directional derivative of $\epsilon_\theta$ in the direction $v \in \mathbb{R}^n$.
\end{proposition}

Because JVP correspond exactly to directional derivatives (see Proposition \ref{prop:jvp}), they can be computed using the forward mode AD where we propagate both the intermediate states variables $z_i$ of the computational graph (chain rule) and the partial derivatives $\frac{\partial z_{i+1}}{\partial z_i}$, similarly to Hessian-vector product (HVP) approaches \cite{blogICLR, Hu_2024}.

\subsection{A family of Jacobian derived regularizations}
This efficient way of computing Jacobian norms can now be used to define a variety of constraints (or losses) to enforce during training. We define two distinct cases of experimentation which correspond to real cases scenarios. The first is when the true dynamic $F$ is known, for example in the case of known equations discretized into a numerical model. The goal of the Neural ODE is then to avoid certain biases or errors coming from this discretization. The second case is when the dynamic is totally unknown, that is to say we only have access to trajectory points but not to $F$, which is much more challenging for the model. 

\paragraph{Known dynamics}
 If we have access to the true dynamics function $F$ (for example known equations, numerical model), we can directly minimize the difference between the directional derivatives $\Vert J_{F_\theta-F}(x)\cdot v_i\Vert_2^2$ estimated through Monte-Carlo  based on the Hutchinson estimator presented before:
\begin{equation}
    \mathcal{L}_{\text{AD}}(\theta,v_i,x)   = \Vert {\hat{D}^{\text{AD}}}_{v_i}(F_\theta-F)(x)\Vert_2^2
\end{equation}
and 
\begin{equation}
     \mathcal{L}_{\text{AD}}(\theta,\{v_i\}_{i=1}^V)= \frac{1}{V}\sum_{x \in \Omega}\sum_{i=1}^V \mathcal{L}_{\text{AD}}(\theta,v_i,x).
\end{equation}
by computing the JVP (or directional derivative) $\hat{D}^{\text{AD}}$ using Automatic Differentiation, in the directions $\{v_i\}_{i=1}^V$ \textit{i.i.d.} sampled from $\mathcal{N}(0,I)$. In practice when we have to choose the training points $x$, we can assume that enforcing the loss on the trajectory points $x_t \in \Omega = \mathcal{T}_F$ (observed or obtained by integrating the true dynamics $F$) is enough to produce accurate or robust trajectories at inference. Although the estimation of the JVP is exact here, we need access to the true dynamics $F$ which can be unknown or relatively heavy to compute (especially for PDEs). 

\paragraph{Unknown dynamics} 
In the case of unknown dynamics, we leverage Finite Differences (FD) approximations of the directional derivatives to compute an unsupervised loss. We first write a first order approximation of the directional derivatives on trajectory points:
\begin{equation}
    \mathcal{L}_{\text{FD}}(\theta, x_{t_i})=\frac{\Vert(F_\theta-F)(x_{t_i+\Delta t})-(F_\theta-F)(x_{t_i})\Vert_2^2}{\Vert x_{t_i+\Delta t}-x_{t_i}\Vert_2^2} 
\end{equation}
and 
\begin{equation}
 \mathcal{L}_{\text{FD}}(\theta)=\frac{1}{N}\sum_{X \in \mathcal{T}_F}\sum_{i=0}^N\mathcal{L}_{\text{FD}}(\theta, x_{t_i}).
\end{equation}
This is an approximation of the directional derivative in the directions of the trajectory. Indeed, if we compute the derivative $F(x(t_i))$ using finite differences, we can write $x_{t_i+\Delta t} =x_{t_i} + F(x_{t_i})\Delta t $ and we get:
\begin{equation}
    \begin{aligned}
 \frac{\Vert(F_\theta-F)(x_{t_i}+ F(x_{t_i}) \Delta t )-(F_\theta-F)(x_{t_i})\Vert_2^2}{\Vert F(x_{t_i})\Vert_2^2 \Delta t^2} \\
         \rightarrow_{\Delta t \rightarrow 0} 
         \Vert D_{\frac{F(x_{t_i})}{\Vert F(x_{t_i})\Vert_2}}(F_\theta-F)(x_{t_i})\Vert_2^2
    \end{aligned}
\end{equation}
where $v_i=\frac{F(x_{t_i})}{\Vert F(x_{t_i})\Vert_2}=\frac{x_{t_i+\Delta t}- x_{t_i}}{\Vert x_{t_i+\Delta t}- x_{t_i}\Vert_2}$ are the normalized directions of the trajectories and $h=\Vert F(x_{t_i})\Vert_2\Delta t$, see notation in equation \ref{eq:dd}. 

Combined with the finite differences approximation $F(x_{t_i+\Delta t})-F(x_{t_i}) = \frac{x_{t_i+2\Delta t}-2x_{t_i+\Delta t}+x_{t_i}}{\Delta t}$, we reach a fully unsupervised version of $\mathcal{L}_{\text{FD}}$. No evaluation of the true dynamics is required to approximate the same quantity as $\mathcal{L}_\text{AD}$. 


\begin{figure*}[t]  
  \centering
  \includegraphics[width=0.33\textwidth]{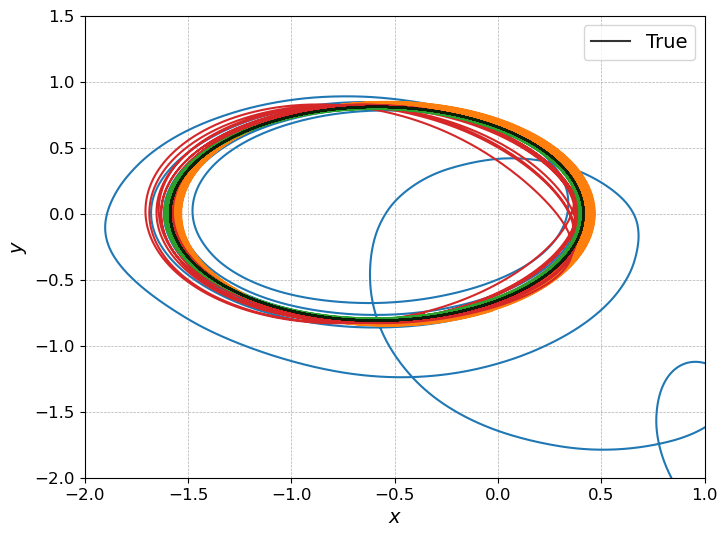}
  \includegraphics[width=0.33\textwidth]{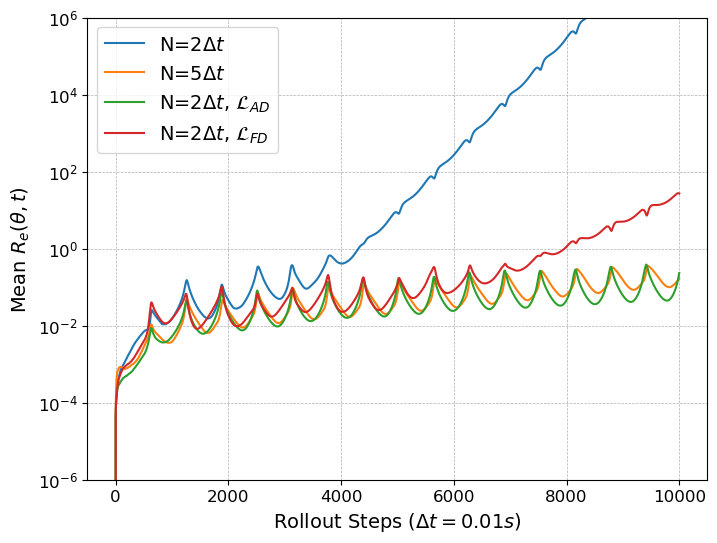}
  \includegraphics[width=0.33\textwidth]{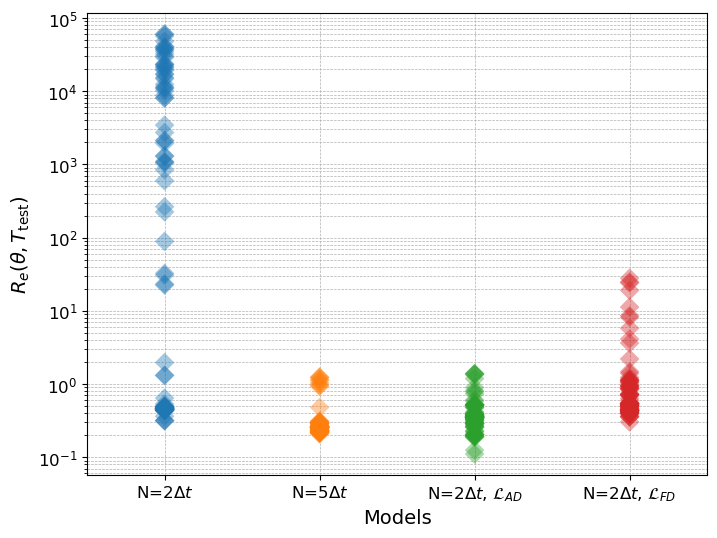}

  \includegraphics[width=0.33\textwidth]{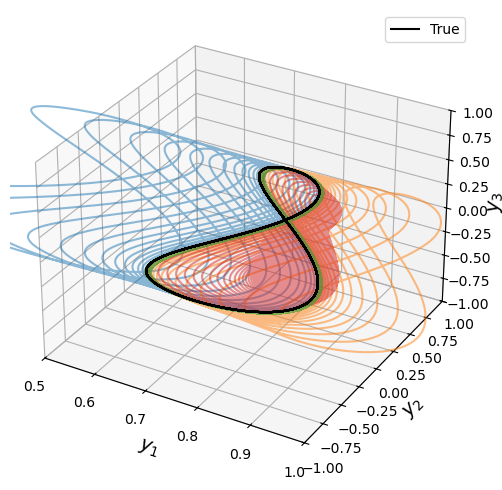}
  \includegraphics[width=0.33\textwidth]{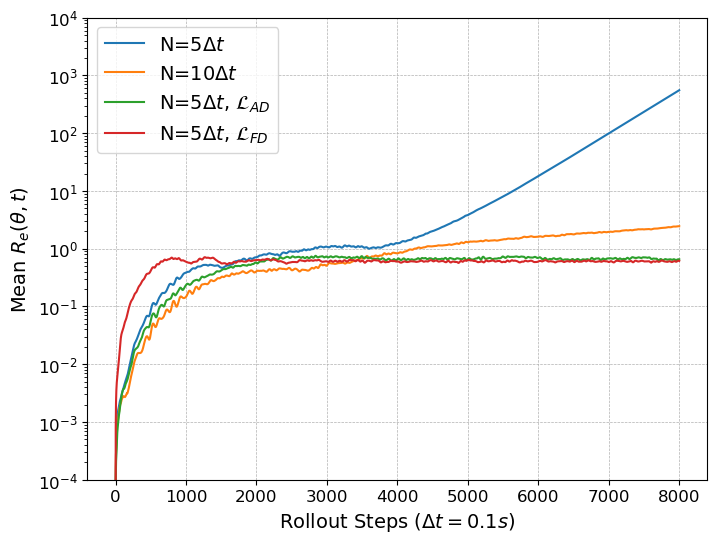}
  \includegraphics[width=0.33\textwidth]{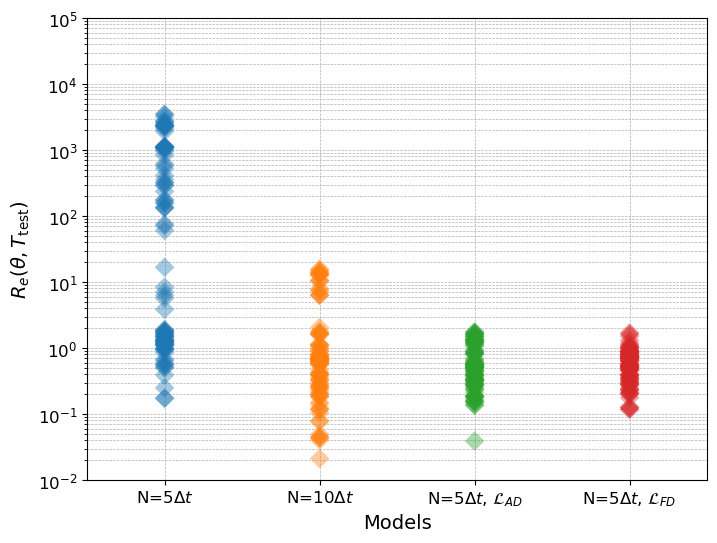}

  \includegraphics[width=0.35\textwidth]{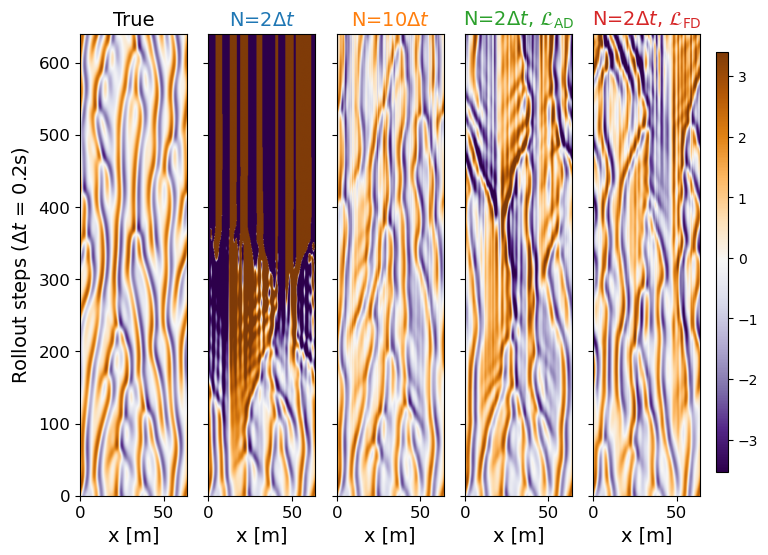}
  \includegraphics[width=0.32\textwidth]{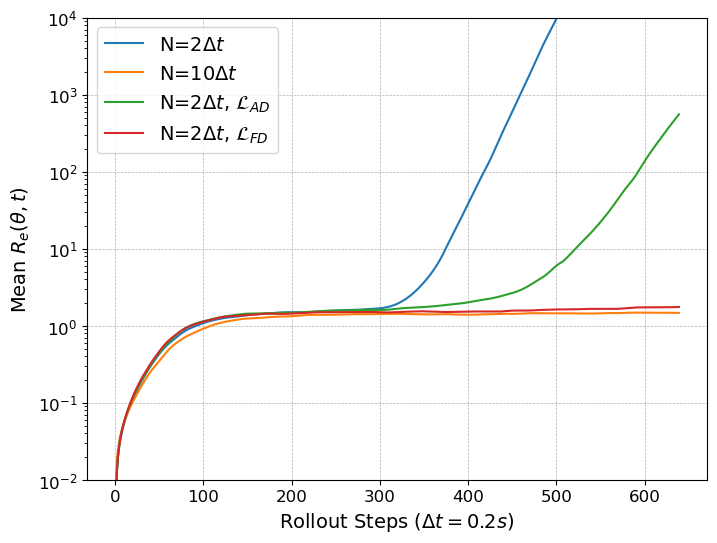}
  \includegraphics[width=0.32\textwidth]{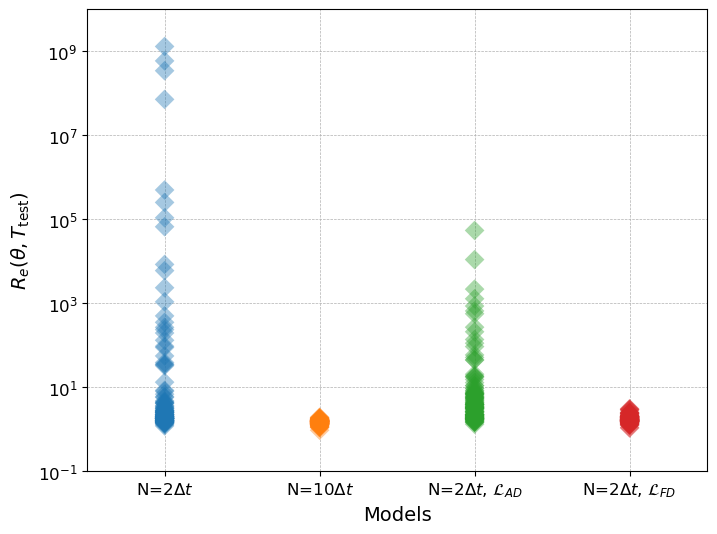}

  \caption{From the left to the right we have for each row: an example of a trajectory, the evolution of the relative error of the state over time and the distribution final values of $R_e(\theta,T_{\text{test}})$. \textbf{Top row:} Two-Body Problem. \textbf{Middle row:} Rigid Body Problem. $N=5\Delta t$ and $N=10\Delta t$ trajectories are cut at 2000 steps for visualization purposes as they drift further away in time. \textbf{Bottom row:} Kuramoto-Sivashinksy.   }
  \label{fig:res}
\end{figure*}

\section{Results}
\label{sec:results}
We test our regularizations on common one-dimensional ODE and PDE cases treated in other papers that cover a variety of dynamics: a first and a second order non-chaotic autonomous systems \cite{white2024stabilizedneuraldifferentialequations} as well as a chaotic fourth-order partial derivative equation \cite{Linot_2023, lippe2023pderefinerachievingaccuratelong}. All present either conservation laws or holonomic constraints that should be preserved during integration. 

In our experiments, we first diagnose the difference between a Neural ODE trained on very short rollouts (producing exploding trajectories in the long-term) and one trained on longer rollouts (stable), with the classical trajectory loss (\ref{eq:traj_loss}). We then apply our Jacobian regularizations during training in the unstable setting (short rollouts), in order to show the improvement that our losses provide on the long-term stability of the system at inference. All methods share the same integration method and time-step $\Delta t$ in order to facilitate their comparison

As in White \yrcite{white2024stabilizedneuraldifferentialequations}, we are interested in the relative error between the predicted state $x_\theta(t)$ and the ground truth $x(t)$ over time, averaged over many trajectories with various initial conditions, referred as $R_e(\theta,t)$. In addition, we analyze the impact of our regularizations on the learned $F_\theta$ and its Jacobian $J_{F_\theta}$ by comparing them to the true $F$ and $J_F$. We focus on the trajectory error (MSE) as well as the offline error $\epsilon_\theta=F-F_\theta$ and jacobian error $J_{e}(\theta)=J_F-J_{F_\theta}$. As our losses are not only comparing trajectories but also penalizing the difference between the learned Jacobian and the one of the system, we expect our models to recover the jacobian properties of the Neural ODE trained on longer rollouts. 

More implementation details on the architecture and training setup are available in appendix \ref{a:implementation}, as well as additional results on conservation laws in \ref{a:res_annex}. 

\subsection{Two-Body Problem}
The Two-Body (TB) Problem describes the motion of two bodies attracting each other with a force inversely proportional to their squared distance, such as the gravitational force \cite{hairer_geometricnumericalintegration}. It is described by the following system of equations:
\begin{equation}
U=\begin{pmatrix}
    x \\
    y \\
    \dot{x} \\
    \dot{y}
\end{pmatrix}, \dot{U} = \begin{pmatrix}
    \dot{x} \\
    \dot{y} \\ 
    \frac{-x}{(x^2+y^2)^{3/2}} \\
    \frac{-y}{(x^2+y^2)^{3/2}}
\end{pmatrix}
\end{equation}
with a body fixed at the origin and $(x,y)$ the Cartesian coordinates of the other body in the orbit plane \cite{white2024stabilizedneuraldifferentialequations}. The angular momentum $L_0$ is conserved here: $x\dot{y}-\dot{x}y=L_0$ for $(x,y)\in \mathcal{T}_F^2$.

As in White \yrcite{white2024stabilizedneuraldifferentialequations} experiments, we train on 40 trajectories with initial conditions $X_0=(1-e, 0, 0, \sqrt{\frac{1 + e}{1 - e}})$ with the eccentricity $e$ sampled from $\mathcal{U}(0.5,0.7)$. To reach a challenging context for the Neural ODE, we use data from trajectories of $T_{\text{train}}=8s$ which is less than a full period, integrated from $F_{\text{TB}}$ with a time-step $\Delta t=0.01s$. We then divide the training trajectories into non-overlapping chunks of $N=2\Delta t$ length. The ground truth and the models both use fourth order Runge-Kutta (RK4) with the same time-step $\Delta t$ as integration scheme to produce their trajectories. 

We then evaluate the long-term stability of our models on 100 test trajectories of duration $T_{\text{test}}=100s$. In the first row of Figure \ref{fig:res}, we see that in the case of known dynamics (green curves), the loss $\mathcal{L}_{\text{AD}}$ allows the model to reach similar performances to those of the Neural ODE trained with a longer rollout $N=5\Delta t$ (orange curves) in terms of relative error $R_e(\theta)$ and final error too. In the case of unknown dynamics, the loss $\mathcal{L_{\text{FD}}}$ also succeeds in stabilizing the trajectories until approximately 7000 steps (red curves), compared to less than 4000 steps without the regularization (blue curves). As expected, this loss is less effective (because the directional derivatives are only approximated here whereas they are exactly computed for $\mathcal{L}_{\text{AD}}$), although it still contains most of the divergence observed in the baseline, only a few trajectories being responsible for the final divergence as observed in the right panel. The conclusions are similar for the relative conservation error, available in Figure \ref{fig:cons} in appendix \ref{a:res_annex}. 

\subsection{Rigid Body Problem}
\label{RB}

The Rigid Body (RB) equations describe the evolution of the angular momentum vector $X$ of a rigid body under Euler's motion laws \cite{hairer_geometricnumericalintegration}: 
\begin{equation}
    U = \begin{pmatrix}
    y_1 \\
    y_2 \\
    y_3
\end{pmatrix}, \dot{U} = \begin{pmatrix}
    0 & -y_3 & y_2 \\
    y_3 & 0 & -y_1 \\
    -y_2 & y_1 & 0 
\end{pmatrix} 
\begin{pmatrix}
    y_1/I_1 \\
    y_2/I_2 \\
    y_3/I_3
\end{pmatrix}
\end{equation}
with $I_1, I_2, I_3$ the principal moments of inertia, where the coordinate axes are the principal axes of the body and the origin of the coordinates fixed at the center of mass of the body \cite{white2024stabilizedneuraldifferentialequations}. The Casimir function $C(U)=\frac{1}{2}(y_1^2+y_2^2+y_3^2$) is conserved along the trajectory and constitutes a holonomic constraint. 

Similarly to White \yrcite{white2024stabilizedneuraldifferentialequations}, we train on 40 trajectories with initial conditions $X_0=(\cos{\phi}, 0, \sin{\phi})$ with $\phi$ sampled from $ \mathcal{U}(0.5,1.5)$. We use data from trajectories of $T_{\text{train}}=15s$ integrated from $F_{\text{RB}}$ with a time-step $\Delta t=0.01s$. We then divide these trajectories into non-overlapping chunks of $N=5\Delta t$ length. The ground truth and the models both use RK4 but with different time-steps: $\Delta t_{\text{model}}=10\Delta t=0.1s$ for faster training.

We then evaluate the long-term stability of our models on 100 test trajectories of duration $T_{\text{test}}=800s$. Looking at the second row of Figure \ref{fig:res}, we see that even the training with a longer rollout $N=10\Delta t$ (orange curves) produces simulations that start diverging after 4000 steps, due to a few exploding trajectories (right panel). In the case of known dynamics ($\mathcal{L_{\text{AD}}}$), we manage to stabilize all trajectories (green curves) and even outperform the $N=10\Delta t$ model. For the case of unknown dynamics (red curves), $\mathcal{L_{\text{FD}}}$ also manages to prevent divergence, however when looking at individual trajectories we observe that the learned dynamic is faulty and sometimes produces collapsing trajectories. The conclusions are the same for the Casimir constraint (see Figure \ref{fig:cons}), with indeed a less good conservation for the $\mathcal{L_{\text{FD}}}$ regularization.

\subsection{Kuramoto-Sivashinsky}

The 1D Kuramoto-Sivashinsky (KS) equation is a fourth order chaotic PDE describing the instabilities of a laminar flame front:
\begin{equation}
    U_t + U_{xx}+ U_{xxxx}+ UU_x=0
\end{equation}
Apart from being non linear, the biharmonic term is particularly difficult to model. Along a trajectory, we have $\int_0^LU(x,t)=0$, $L$ being the length of the spatial domain with periodic boundary conditions.  

To generate ground truth data, we follow the setups of \citeauthor{brandstetter2022liepointsymmetrydata} \yrcite{brandstetter2022liepointsymmetrydata} and \citeauthor{lippe2023pderefinerachievingaccuratelong} \yrcite{lippe2023pderefinerachievingaccuratelong}, using a uniform discretization of the domain into 256 points. We do not sample across different $L$ and $\Delta t$ but keep them fixed at $L=64, \Delta t=0.2$, and generate 512 training trajectories of duration $T_{\text{train}}=140\Delta t$. As in \citeauthor{lippe2023pderefinerachievingaccuratelong} \yrcite{lippe2023pderefinerachievingaccuratelong}, the first 360 steps are truncated and considered as a warmup for the solver (i.e., $t_0=360\Delta t$). The initial conditions are sampled from a distribution over truncated Fourier series with random
coefficients $\{A_m,l_m,\phi_m\}$ as $U_0(x) = \sum_{m=1}^{10} A_m \sin(2\pi l_mx/L+ \phi_m)$. The integration of the ground truth data is done with an implicit Runge-Kutta integration method of the Radau IIA family (order 5) \cite{brandstetter2022liepointsymmetrydata}, and our models use the RK4 solver with the same $\Delta t$. Finally, as for our previous experiments, we divide the training trajectories into non-overlapping chunks of $N=2\Delta t$ steps and evaluate the long-term stability on 128 test trajectories of duration $T_{\text{test}}=640\Delta t$.

In the bottom row of Figure \ref{fig:res}, we observe that both methods also succeeds in limiting ($\mathcal{L_{\text{AD}}}$) or totally preventing ($\mathcal{L_{\text{FD}}}$) diverging trajectories. Surprinsingly, the $\mathcal{L_{\text{AD}}}$ regularization which was the best for the previous experiments, is now outperformed by the $\mathcal{L_{\text{FD}}}$ one. This behavior is a sign of the loss being hard to tune, for reasons that we explain in the next section \ref{table_res}. Nevertheless, it succeeds in delaying the divergence time as well as reducing the intensity of it. Similar results for mass conservation are found in the appendix \ref{a:res_annex}. The artifacts (lines) observed in all simulations are expected when using MLP on this task and could be removed with a better adapted architecture such as a Convolutional Neural Network \cite{Linot_2023}, although it is not what is evaluated here, but rather the explosion of errors over time. 

\subsection{Impact on learned $F_\theta$ and $J_{F_\theta}$} 
\label{table_res}

Table \ref{tab:lip_diff} summarizes the evaluation metrics values of the different experiments and methods. For $\epsilon_\theta$ and $J_e(\theta)$, we report their differences with the metrics $\epsilon_\theta^B$ and $J_e^B(\theta)$ of the short rollout Neural ODE baseline (blue) to highlight the effect induced by each method on them. Their raw values are available in Table \ref{tab:lip} in the annexes. 

\begin{table}[h]
  \caption{Characteristics of the learned $F_\theta$ and $J_{F_\theta}$ compared to the true ones. Lower is best, with best and second best results respectively in bold and underlined.}
  \label{tab:lip_diff}
  \begin{center}
    \begin{small}
      \begin{sc}
        \begin{tabular}{lcc|cc}
          \toprule
          Exp & $F_\theta$ & $\mathcal{L_{\text{traj}}}(\theta)$ &
          $\epsilon_\theta - \epsilon_\theta^B$ &
          $J_{e}(\theta)-J_e^B(\theta)$  \\
          \midrule

          \multirow{5}{*}{TB}
          & \textcolor{Cerulean}{N=2$\Delta t$} & 4.9e15 & 0 & 0  \\
          & \textcolor{orange}{N=5$\Delta t$} & \underline{1.7e-2} & +3.4e-6 & \textbf{-0.15}    \\
          & \textcolor{LimeGreen}{N=2$\Delta t$, $\mathcal{L}_{\text{AD}}$} & \textbf{1.0e-2} & \underline{-2.8e-7} & \underline{-0.03} \\
          & \textcolor{red}{N=2$\Delta t$, $\mathcal{L}_{\text{FD}}$} & 2.64e2  & \textbf{-3.1e-7} & -0.02  \\
          \midrule

          \multirow{5}{*}{RB}
          & \textcolor{Cerulean}{N=5$\Delta t$} & 1.41e4  & \underline{0} & 0   \\
          & \textcolor{orange}{N=10$\Delta t$} & 1.62e0 & +0.79e-7 & -0.55e-3  \\
          & \textcolor{LimeGreen}{N=5$\Delta t$, $\mathcal{L}_{\text{AD}}$} & \underline{1.79e-1} & \textbf{-0.13e-7} & \textbf{-7.39e-3} \\
          & \textcolor{red}{N=5$\Delta t$, $\mathcal{L}_{\text{FD}}$} & \textbf{1.52e-1} & +4.73e-6 & \underline{-1.96e-3}  \\
          \midrule

          \multirow{5}{*}{KS}
          & \textcolor{Cerulean}{N=2$\Delta t$} & 4.77e14 & 0 & 0  \\
          & \textcolor{orange}{N=10$\Delta t$} & \textbf{2.91e0} & +5.49e-3 & \textbf{-2.15e-2}  \\
          & \textcolor{LimeGreen}{N=2$\Delta t$, $\mathcal{L}_{\text{AD}}$} & 8.50e5 & \underline{-0.04e-3} & -0.01e-2  \\
          & \textcolor{red}{N=2$\Delta t$, $\mathcal{L}_{\text{FD}}$} & \underline{3.59e0} & \textbf{-0.08e-3} & \underline{-0.05e-2}  \\
          \bottomrule
        \end{tabular}
      \end{sc}
    \end{small}
  \end{center}
  \vskip -0.1in
\end{table}

When first comparing the Neural ODE trained only with the trajectory loss, we observe that the ones trained on shorter rollouts (blue) perform badly in terms of trajectory loss despite having a smaller offline error $\epsilon_\theta$, whereas the ones trained on longer rollouts (orange) have a smaller trajectory loss although having a bigger offline error, as expected. Our novel diagnosis on the Jacobians reveals that, consistenly across all experiments, the model trained on longer rollouts has a lower Jacobian error $J_e(\theta)$ than the one trained on shorter rollouts. This highlights the importance of correctly learning the Jacobian for long-term integration, which seems to be naturally improved when training on longer rollouts. 

While being trained on short rollouts, our regularized Neural ODEs reproduce this same trend on the Jacobian errors: all are smaller than $J_e^B(\theta)$, although usually with smaller magnitudes than the Neural ODE trained on longer rollouts. This can be explained by the different effect our regularizations have on the offline error, which is almost always improved (except for the $\mathcal{L}_{\text{FD}}$ experiment on Rigid Body where the learned dynamic sometimes produces collapsing trajectories as seen in section \ref{RB}). This shows that the balance between $\epsilon_\theta$ and $J_e(\theta)$ changes with the training rollout, with shorter rollouts also needing a smaller offline to produce stable long-term inferences when longer rollouts only need a smaller $J_e(\theta)$. 

Both of our regularizations succeed in finding this balance on short rollouts and prevent diverging trajectories. The $\mathcal{L}_{\text{AD}}$ loss is the most successful for the Two-Body and Rigid Body experiments, with best or second best metrics in Table \ref{tab:lip_diff}, and also comparable or even better than the longer-rollout baselines. The $\mathcal{L}_{\text{FD}}$ regularization also obtains good results but can be too restrictive (RB experiment). This was expected as $\mathcal{L}_{\text{AD}}$ is a far more informative regularization, having access to the true $J_F$. 

However this loss also has its limitations, as shown with the KS experiment where it fails in preventing all divergence. This can be explained by the Jacobian $J_F$ (involving Fourier Transforms) becoming very large compared to the Jacobian of a Neural Network, making our loss $L_{AD}$ harder to tune. On the contrary, since $\mathcal{L}_{\text{FD}}$ compares similar quantities (finite differences), the loss is easier to tune, less costly to compute and yields better results, performing almost as well as the $N=10\Delta t$ model in terms of trajectory loss. 

\section{Conclusion}
\label{sec:conclu}

In this paper, we successfully stabilized the long-term integration of Neural ODEs trained on short rollouts by regularizing the Jacobian of the learned dynamic. We proved the efficiency of our regularizations in both cases of known and unknown dynamics and across ODE and PDEs of various dimensions. In particular the $\mathcal{L}_{\text{FD}}$ loss has a low computational cost and can be easily scaled to higher dimensional problems, although it usually does not produce as accurate trajectories as the $\mathcal{L}_{\text{AD}}$ regularization. 

We showed that our Jacobian loss is key in the success of online training and can be used to reduce the length of the training rollout without losing stability. This study is a step forward in understanding the balance between the error between the dynamic functions, the error between their Jacobians and the training rollout, which could be of interest when the solver we use is only partially differentiable \cite{list2024differentiabilityunrolledtrainingneural} as in current climate models. Finally, our methods could be directly applied to hybrid settings where the Neural ODE only replaces a part of the dynamic. 

\section*{Impact Statement}

This paper presents work whose goal is to advance the field
of Machine Learning. There are many potential societal
consequences of our work, none which we feel must be
specifically highlighted here. 

\bibliography{biblio}
\bibliographystyle{icml2026}

\newpage
\appendix
\onecolumn
\section{Lipschitz constant and trajectory error}
\label{a:lipschitz}
\textit{In this section, we denote $\Vert \cdot \Vert_2$ as $\Vert \cdot \Vert$ and $L_{F_\theta}$ as $L_\theta$ for lighter notations.}

Let us suppose we have two dynamics, the true equation with exact resolution $x$ and dynamics $F$, the other $x_\theta$ from the learned dynamic $F_\theta$:
\begin{equation}
\label{eq:basis}
\left\{
    \begin{array}{ll}
        \dot{x}(t) = F(x(t)) \\
        x(t=0) = x_0
    \end{array}
\right. \quad \text{and} \quad
\left\{
    \begin{array}{ll}
        \dot{x_\theta}(t) = F_\theta(x_\theta(t))  \\
        y(t=0) = x_0
    \end{array}
\right.
\end{equation}

We make the assumption that $F$ and $F_\theta$ are, respectively, L and $L_\theta$-Lipschitz. We define $\epsilon_\theta(x(t))=F(x(t))-F_\theta(x(t))$ the offline error. We prove the following inequalities:

\begin{proposition}
\begin{equation}
    \forall t \geq 0, \Vert x(t)-x_\theta(t)\Vert \leq \exp(L_\theta t)  \int_0^t \Vert\epsilon_\theta(x(s))\Vert \exp(-L_\theta s)ds,
\end{equation}

and if $\epsilon_\theta$ is bounded by $\Vert\epsilon_\theta\Vert_\infty$:
\begin{equation*}
\begin{aligned}
        \Vert x(t)-x_\theta(t)\Vert &\leq \Vert\epsilon_\theta\Vert_\infty \exp(L_\theta t) \int_0^t  \exp(-L_\theta s)ds   \\
        &\leq \Vert\epsilon_\theta\Vert_\infty \exp(L_\theta t)    \frac{1}{L_\theta}(1-\exp{(-L_\theta t)})  \\
        &\leq    \frac{\Vert\epsilon_\theta\Vert_\infty}{L_\theta}(\exp{(L_\theta t)-1}).
\end{aligned}
\end{equation*}
\end{proposition}

\begin{proof}
The second inequality comes from the first one by integrating the exponential. The first one uses Grönwall's lemma:
\begin{lemma}Grönwall's inequality: \\
\text{Let $\psi\geq 0, \phi, K$ be continuous functions.} \\
If
$$\forall t\geq t_0, \phi(t) \leq K(t) + \int_{t_0}^t{\psi(s)\phi(s) ds},$$ 
then \\
$$\forall t\geq t_0, \phi(t) \leq K(t) + \int_{t_0}^t{K(s)\psi(s) \exp{\left( \int_s^t{\psi(u) du}  \right)} ds}.$$
\end{lemma}

As $x$ and $x_\theta$ are continuous, with the Fundamental theorem of calculus we can write: 
\begin{equation*}
\begin{aligned}
        \Vert x(t)-x_\theta(t)\Vert &= \Vert x_0 + \int_0^t{\dot{x}(s)ds} - x_0 - \int_0^t{\dot{x_\theta}(s)ds}\Vert  \\
        &= \Vert \int_0^t{F(x(s))-F_\theta(x_\theta(s)) ds} \Vert \quad \text{using \ref{eq:basis}} \\
        &= \Vert \int_0^t{F(x(s))- F_\theta(x(s)) + F_\theta(x(s)) - F_\theta(x_\theta(s)) ds} \Vert \quad  \\
        &\leq  \int_0^t{\Vert\epsilon_\theta(x(s))\Vert ds}+  \int_0^t{\Vert F_\theta(x(s))-F_\theta(x_\theta(s))\Vert ds}   \quad \text{(triangular inequality)} \\
        \Vert x(t)-x_\theta(t)\Vert &\leq \int_0^t{\Vert\epsilon_\theta(x(s))\Vert ds} + L_\theta \int_0^t{\Vert x(s)-x_\theta(s)\Vert ds}  \quad \text{as $F_\theta$ is $L_\theta$-Lipschitz.}
\end{aligned}
\end{equation*}

We can apply Grönwall's inequality with $\phi(t) = \Vert x(t)-x_\theta(t)\Vert$, $\psi(t)= L_\theta \geq 0$ and $K(t)= \int_0^t{\Vert\epsilon_\theta(x(s))\Vert ds}$:

\begin{equation*}
\begin{aligned}
        \Vert x(t)-x_\theta(t)\Vert
        &\leq \int_0^t{\Vert\epsilon_\theta(x(s))\Vert ds} + \int_0^t{ L_\theta\ \left( \int_0^s{\Vert\epsilon_\theta(x(u))\Vert du}\right) \exp{\left( \int_s^t{L_\theta du}  \right) ds} }\\
        &\leq K(t) + L_\theta \exp{(L_\theta t)}\int_0^t{ \left( \int_0^s{\Vert\epsilon_\theta(x(u)\Vert du}\right) \exp{(-L_\theta s)}ds} .
\end{aligned}
\end{equation*}

For the integral, we can use integration by parts (IPP), with

\begin{equation*}
\left\{
    \begin{array}{ll}
        h(s)= \int_0^s{\Vert\epsilon_\theta(x(u))\Vert du} =K(s)\\
        \dot{h}(s) = \Vert\epsilon_\theta(x(s))\Vert
    \end{array}
\right. \quad \text{and} \quad
\left\{
    \begin{array}{ll}
        g(s) = -\frac{1}{L_\theta}\exp(-L_\theta s) \\
        \dot{g}(s)=\exp{(-L_\theta s)} 
    \end{array}
\right.
\end{equation*}

We have: 
\begin{equation*}
\begin{aligned}
       \int_0^t{h(s)\dot{g}(s)ds} &= \left[ h(s)g(s) \right]_0^t - \int_0^t{\dot{h(s)}g(s) ds} \\ 
       &= \left[ K(s) \frac{-1}{L_\theta}\exp(-L_\theta s) \right]_0^t - \int_0^t \frac{-1}{L_\theta}\exp(-L_\theta s)\Vert\epsilon_\theta(x(s))\Vert ds \\
       \int_0^t{h(s)\dot{g}(s)ds}&= \frac{-1}{L_\theta}\exp(-L_\theta t) K(t) + \frac{1}{L_\theta}\int_0^t \Vert\epsilon_\theta(x(s))\Vert\exp(-L_\theta s)ds
\end{aligned}
\end{equation*}

Plugging it back to the main inequality,
\begin{equation*}
\begin{aligned}
        |x(t)-y(t)| &\leq K(t) + L_\theta \exp{(L_\theta t)} \left( \int_0^t{\left(\int_0^s{\Vert\epsilon_\theta(x(u))\Vert du}\right)  \exp{(-L_\theta s)}ds} \right) \\
        &\leq K(t) + L_\theta \exp(L_\theta t) \left( \frac{-1}{L_\theta}\exp(-L_\theta t)K(t) + \frac{1}{L_\theta} \int_0^t \Vert\epsilon_\theta(x(s))\Vert \exp(-L_\theta s)ds\right) \\
        &\leq K(t)  -K(t) + \exp(L_\theta t)\int_0^t \Vert\epsilon_\theta(x(s))\Vert \exp(-L_\theta s)ds \\
        &\leq  \exp(L_\theta t)  \int_0^t \Vert\epsilon_\theta(x(s))\Vert \exp(-L_\theta s)ds.
\end{aligned}
\end{equation*}
\end{proof}

\newpage
\section{Implementation details}
\label{a:implementation}
This section presents the different parameters and protocols to reproduce our experiments. All of our experiments are realized in jax, with code available on GitHub. 

\subsection{Data generation} 
For Two-Body and Rigid Body experiments, the ground truth trajectories are integrated using fourth order Runge Kutta (RK4) with time-step $\Delta t=0.01s$. For the Rigid Body Problem, we use the same parameter values $(I_1, I_2, I_3)=(1.6, 1.0, 2/3)$ as in White \yrcite{white2024stabilizedneuraldifferentialequations}.

For Kuramoto-Sivashinsky, we follow the protocol from \citeauthor{lippe2023pderefinerachievingaccuratelong}, using \citeauthor{brandstetter2022liepointsymmetrydata} implementation. We only alter the generation script by fixing $\Delta t=0.2s$ and $L=64$, instead of varying them. The code can then be used to produce our data with the following arguments: 
\begin{itemize}
    \item Training set: \texttt{--train\_samples=512 --nt=500}
    \item Validation set: \texttt{--valid\_samples=128 --nt=500}
    \item Test set:  \texttt{--test\_samples=128 --nt=1000 --nt\_effective=640 --end\_time=200}
\end{itemize}

Finally, the training and validation trajectories are chunked into non-overlapping segments of $N$ length as in White \yrcite{white2024stabilizedneuraldifferentialequations}, in order for all the models to use the exact same points independently of their training rollout. This setup makes the validation particularly challenging, with no information on the long-term performances of the model given at this step.  

\subsection{Experimental setting}
We use the MLP architecture from \citeauthor{Yin_2021} for the Neural ODE models (i.e., for $F_\theta$). It is composed of three linear layers with activation ReLU and hidden size 200. All Neural ODEs use RK4 as integration scheme as well, and are trained using Adam optimizer \cite{kingma2017adammethodstochasticoptimization} with learning rate $l_r=0.001$, $\beta_1=0.9$ and $\beta_2=0.999$. For the $\mathcal{L}_{\text{AD}}$ loss, we sample $V=10$ directions $\{v_i\}_{i=1}^V$ from $\mathcal{N}(0,I)$ for each experiment.  Table \ref{tab:params} summarizes the parameters of the different models and experiments.

\def \LambdTB{\makecell{  1e-11, 5e-12, 1e-12, 5e-13, 1e-13  }}

\def \LambdRB{\makecell{  1e-2, 1e-3, 1e-4,
1e-5 1e-6  }}

\def \LambdKSFD{\makecell{  1e-4, 1e-5, 1e-6,
1e-7, 1e-8  }}

\begin{table}[h]
\caption{Summary of model and experiment parameters}
\label{tab:params}
  \begin{center}
    \begin{small}
      \begin{sc}
    \begin{tabular}{|l|c|c|c|}
    \hline
    \textbf{Parameters} & \textbf{Two-Body} & \textbf{Rigid Body} & \textbf{Kuramoto-Sivashinsky} \\
    \hline
    $U$ dimension  & 4 & 3 & 256 \\
    \hline
    $\Delta t / \Delta t_{\text{model}}$ (s) & 0.01 / 0.01 & 0.01 / 0.1 & 0.2 / 0.2 \\
    \hline
    $T_{\text{train}} / T_{\text{val}} / T_{\text{test}}$ (s) & 8 / 8 / 100 & 15 / 15 / 800 & 28 / 28 / 128   \\
    \hline
    $N_{\text{train}} / N_{\text{val}} / N_{\text{test}}$ (s) & 2 / 2 / 10000 & 5 / 5 / 8000 & 2 / 2 / 640   \\
    \hline
    Nb traj train/val / test & 40 / 40 / 100 & 40 / 40 / 100 & 512 / 128 / 128  \\
    \hline
    Nb epochs & 4500 & 4500 & 1000 \\
    \hline
    $\lambda_{\text{AD}}$ interval & \LambdTB & \LambdRB & 5e-12, 1e-12, 5e-13, 1e-13 \\
    \hline
    $\lambda_{\text{AD}}$ optimal & 5e-13 & 1e-6 & 5e-13 \\
    \hline
    $\lambda_{\text{FD}}$ interval & \LambdTB & \LambdRB & \LambdKSFD \\
    \hline
    $\lambda_{\text{FD}}$ optimal & 5e-13 & 1e-2 & 1e-7 \\
    \hline
    \end{tabular}
\end{sc}
\end{small}
\end{center}
\end{table}

The best model is determined based on the mean trajectory loss on a validation set of chunked trajectories. The optimal regularization coefficient $\lambda$ is determined by grid-search for each experiment, also based on the same metric. Figure \ref{fig:val} shows its evolution across the tested hyperparameters $\lambda$. As the validation is done on very short rollouts, the mean validation MSEs are quite similar. The $\mathcal{L_{\text{AD}}}$ regularization was particularly difficult to tune for the KS experiment, with unstable training when $\lambda\geq$ 1E-10, and could benefit from finer grid-search exploration. 

\begin{figure}[h]
    \centering
    \includegraphics[width=0.33\linewidth]{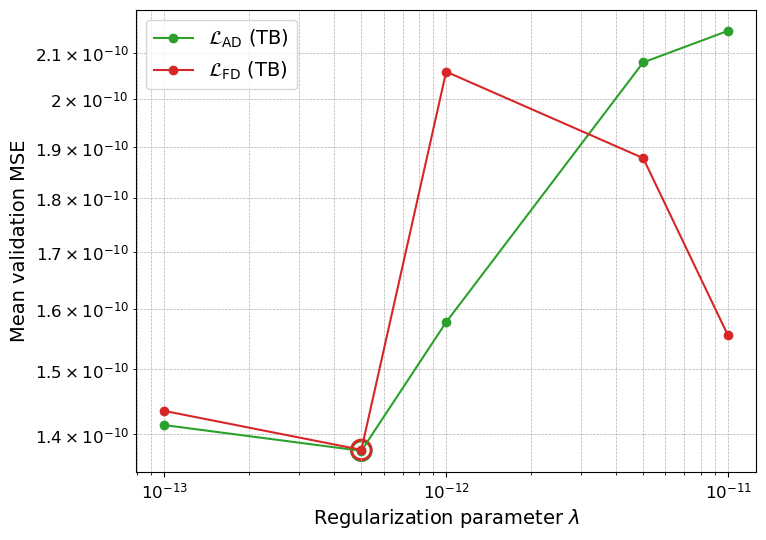}
    \includegraphics[width=0.33\linewidth]{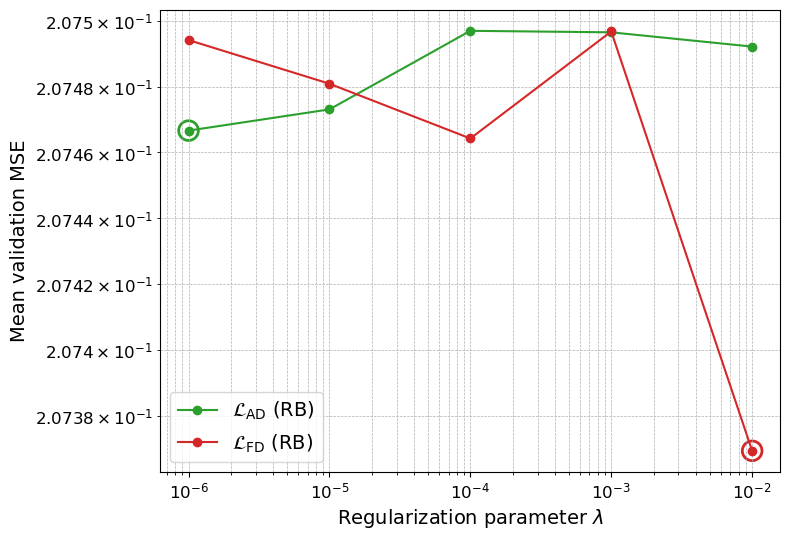}
    \includegraphics[width=0.33\linewidth]{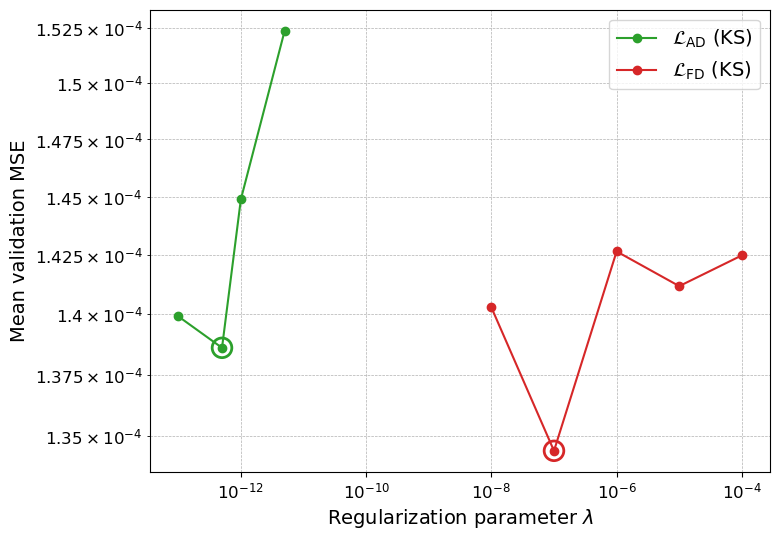}
    \caption{Mean MSE on validation set for the optimization of the $\lambda$ parameter, for the Two-Body Problem (left) and the Rigid Body Problem (right). The best outcome is circled.}
    \label{fig:val}
\end{figure}

\subsection{Evaluation}
We derive all evaluation metrics reported in Table \ref{tab:lip} on the ground truth test trajectories points $x\in \mathcal{T}_F^{\text{test}}$.

For the $J_e(\theta)$ metric, we use the true jacobians obtained by forward-mode AD, column-by-column (jax.jacfwd), except for the Kuramoto-Sivashinsky (KS) experiment where it is too costly to compute and is estimated using the Hutchinson Monte-Carlo described in equation \ref{eq:dd}, with $V=4$ directions sampled from $\mathcal{N}(0,I)$.

\section{Additional results}
\label{a:res_annex}
\begin{table}[h]
  \caption{Characteristics of the learned $F_\theta$ and $J_{F_\theta}$ compared to the true ones. Lower is best, with best and second best results respectively in bold and underlined. As the values of $J_e(\theta)$ for KS are very high, we report up to a constant $C=265560.9$ to highlight their differences.}
  \label{tab:lip}
  \begin{center}
    \begin{small}
      \begin{sc}
        \begin{tabular}{lcccc}
          \toprule
          Exp & $F_\theta$ & $\mathcal{L_{\text{traj}}}(\theta)$ &
          $\epsilon_\theta$ &
          $J_{e}(\theta)$  \\
          \midrule

          \multirow{5}{*}{TB}
          & \textcolor{Cerulean}{N=2$\Delta t$} & 4.9e15 & 1.42e-6 & 5.74  \\
          & \textcolor{orange}{N=5$\Delta t$} & \underline{1.7e-2} & 4.82e-6 & \textbf{5.59}    \\
          & \textcolor{LimeGreen}{N=2$\Delta t$, $\mathcal{L}_{\text{AD}}$} & \textbf{1.0e-2} & \underline{1.14e-6} & \underline{5.71} \\
          & \textcolor{red}{N=2$\Delta t$, $\mathcal{L}_{\text{FD}}$} & 2.64e2  & \textbf{1.11e-6} & 5.72  \\
          \midrule

          \multirow{5}{*}{RB}
          & \textcolor{Cerulean}{N=5$\Delta t$} & 1.41e4  & \underline{4.22e-7} & 7.84e-3   \\
          & \textcolor{orange}{N=10$\Delta t$} & 1.62e0 & 5.01e-7 & 7.29e-3  \\
          & \textcolor{LimeGreen}{N=5$\Delta t$, $\mathcal{L}_{\text{AD}}$} & \underline{1.79e-1} & \textbf{4.09e-7} & \textbf{4.51e-4} \\
          & \textcolor{red}{N=5$\Delta t$, $\mathcal{L}_{\text{FD}}$} & \textbf{1.52e-1} & 5.16e-6 & \underline{5.88e-3}  \\
          \midrule

          \multirow{5}{*}{KS}
          & \textcolor{Cerulean}{N=2$\Delta t$} & 4.77e14 & 2.51e-3 & C+5.63e-2  \\
          & \textcolor{orange}{N=10$\Delta t$} & \textbf{2.91e0} & 8.00e-3 & C+\textbf{3.48e-2}  \\
          & \textcolor{LimeGreen}{N=2$\Delta t$, $\mathcal{L}_{\text{AD}}$} & 8.50e5 & \underline{2.47e-3} & C+5.62e-2  \\
          & \textcolor{red}{N=2$\Delta t$, $\mathcal{L}_{\text{FD}}$} & \underline{3.59e0} & \textbf{2.43e-3} & C+\underline{5.58e-2}  \\
          \bottomrule
        \end{tabular}
      \end{sc}
    \end{small}
  \end{center}
  \vskip -0.1in
\end{table}

\begin{figure}[h]
    \centering
    \includegraphics[width=0.49\linewidth]{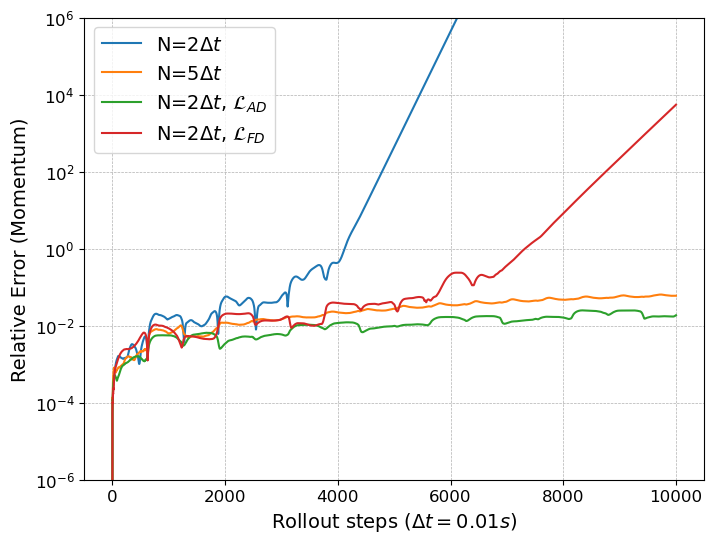}
    \includegraphics[width=0.49\linewidth]{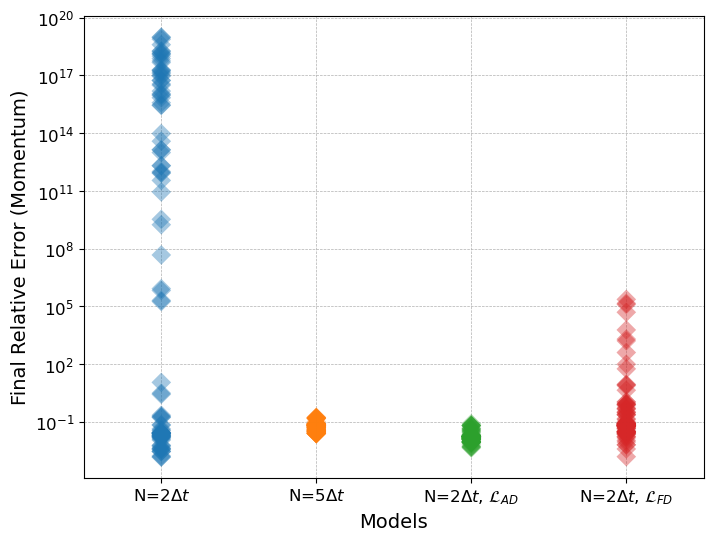}
    
    \includegraphics[width=0.49\linewidth]{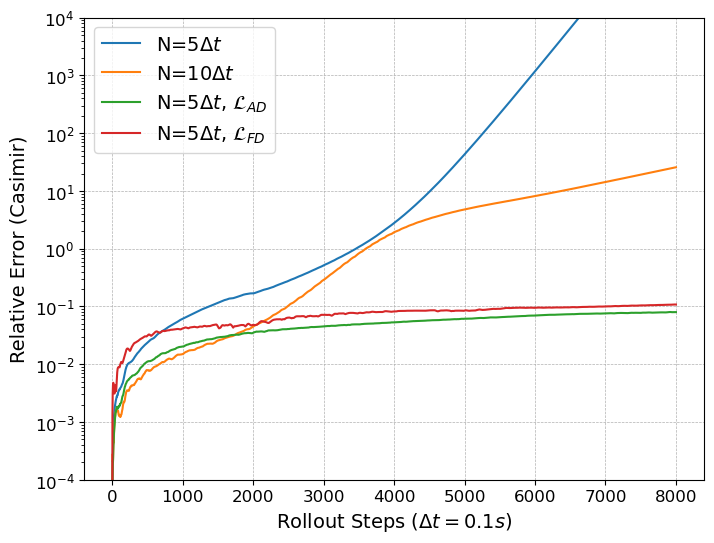}
    \includegraphics[width=0.49\linewidth]{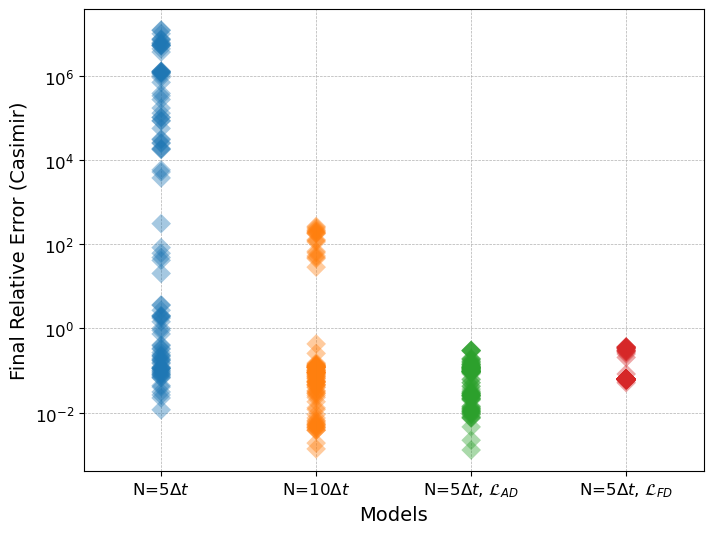}

    \includegraphics[width=0.49\linewidth]{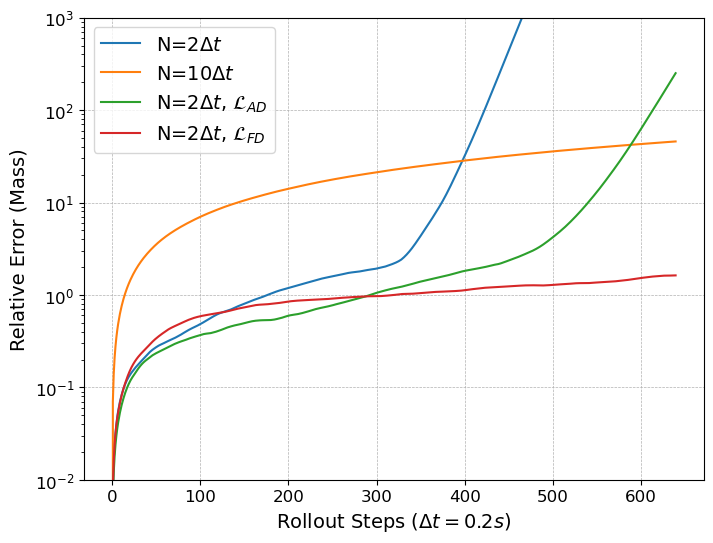}
    \includegraphics[width=0.49\linewidth]{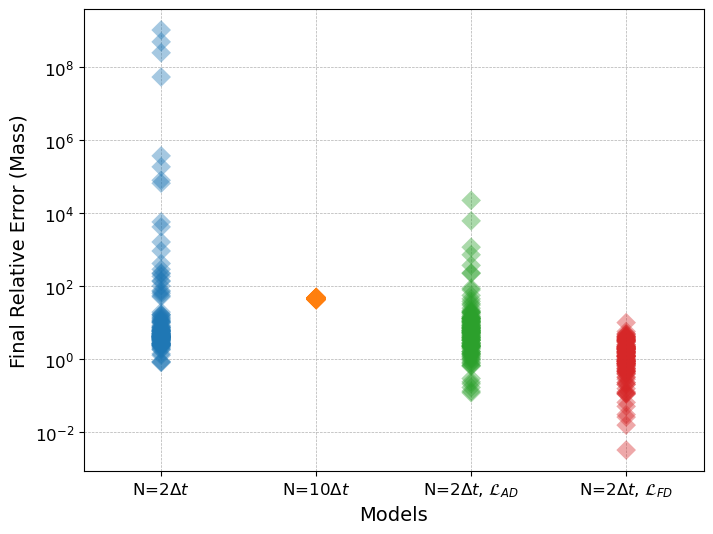}
    \caption{Mean Relative Error of the conserved properties over time (left) and distribution of final values (right). \textbf{Top row:} Two-Body Problem. \textbf{Middle row:} Rigid Body Problem. \textbf{Bottom row}: Kuramoto-Sivashinsky}
    \label{fig:cons}
\end{figure}


\end{document}